\documentclass[twocolumn]{article}
\usepackage{mglgrid}

\usepackage[utf8]{inputenc} 
\usepackage[T1]{fontenc}    
\usepackage[hidelinks]{hyperref}       
\usepackage{url}            
\usepackage{booktabs}       
\usepackage{amsfonts}       
\usepackage{nicefrac}       
\usepackage{microtype}      
\usepackage{xcolor}         

\usepackage{centernot}
\usepackage[title]{appendix}
\usepackage{pgfplots}
\pgfplotsset{compat = newest}
\makeatletter
\def\pgfplotsforeachtodomain@@.#1\relax{%
    \pgf@xa=0.#1pt %
    \ifdim\pgf@xa>0.01pt  
        \def\pgfplotsretval{}%
        \def\pgfplotsretvalb{}%
    \else
        \edef\pgfplotsretvalb{\the\c@pgf@counta}%
    \fi
}%
\makeatletter

\usetikzlibrary{arrows.meta}
\usepackage{graphicx}
\usepackage{booktabs}
\usepackage{multirow}
\usepackage{xspace}
\usepackage{booktabs}
\usepackage{float}
\usepackage{color}
\usepackage{mathtools}

\usepackage{amsthm}
\usepackage{bbm}
\usepackage{thmtools}
\usepackage{thm-restate}
\usepackage{xcolor}
\usepackage[normalem]{ulem}
\usepackage{siunitx}
\usepackage{makecell}
\usepackage[noend]{algpseudocode}
\usepackage{algorithm}
\newcommand*\Let[2]{\State #1 $\gets$ #2}
\makeatletter%
\if@twocolumn%
  \algnewcommand{\MyComment}[1]{}
\else%
  \algnewcommand{\MyComment}[1]{\Comment{#1}}
\fi
\makeatother
\usepackage[font=small]{caption}
\usepackage{subcaption}
\usepackage{tabularx}
\usepackage{cleveref}

\DeclareMathOperator*{\argmin}{arg\,min}
\DeclarePairedDelimiter{\ceil}{\lceil}{\rceil}
\DeclarePairedDelimiter{\floor}{\lfloor}{\rfloor}
\DeclareMathOperator*{\avg}{avg}
\newcommand*{\SP}{\mathrm{SP}}
\newcommand*{\natnum}{\mathbb{N}}
\newcommand*{\natnumzero}{\mathbb{N}_0}

\newtheorem{definition}{Definition}
\newtheorem{theorem}{Theorem}
\newtheorem{proposition}[theorem]{Proposition}

\newtheorem{assumption}{Assumption}

\newcommand{\divides}{\mathrel{|}}
\newcommand{\Nb}{\mathbb{N}}
\newcommand{\Rb}{\mathbb{R}}
\newcommand{\OL}{\mathcal{O}}
\newcommand{\tta}{\liningnums{2}TA}

\makeatletter
\newcommand{\oset}[3][0ex]{%
  \mathrel{\mathop{#3}\limits^{
    \vbox to#1{\kern-3\ex@
    \hbox{$\scriptstyle#2$}\vss}}}}
\makeatother
\newcommand{\pto}{\oset{p}\to}

\allowdisplaybreaks

\ifdef{\groundskip}
      {\title{Two-Tailed Averaging}
       \subtitle{Anytime, Adaptive, Once-in-a-While Optimal\\ Weight Averaging for Better Generalization}}
      {\title{Two-Tailed Averaging: Anytime, Adaptive,\\Once-in-a-While Optimal Weight Averaging\\for Better Generalization}}

\author{\textit{G\'abor Melis} \\
  {\tt melisgl@google.com}\\
  DeepMind, UCL; London, UK
}

\begin{document}

\maketitle

\begin{abstract}
Tail Averaging improves on Polyak averaging's non-asymptotic behaviour by excluding a number of leading iterates of stochastic optimization from its calculations.
In practice, with a finite number of optimization steps and a learning rate that cannot be annealed to zero, Tail Averaging can get much closer to a local minimum point of the training loss than either the individual iterates or the Polyak average.
However, the number of leading iterates to ignore is an important hyperparameter, and starting averaging too early or too late leads to inefficient use of resources or suboptimal solutions.
Our work focusses on improving generalization, which makes setting this hyperparameter even more difficult, especially in the presence of other hyperparameters and overfitting.
Furthermore, before averaging starts, the loss is only weakly informative of the final performance, which makes early stopping unreliable.
To alleviate these problems, we propose an anytime variant of Tail Averaging intended for improving generalization not pure optimization, that has no hyperparameters and approximates the optimal tail at all optimization steps.
Our algorithm is based on two running averages with adaptive lengths bounded in terms of the optimal tail length, one of which achieves approximate optimality with some regularity.
Requiring only the additional storage for two sets of weights and periodic evaluation of the loss, the proposed Two-Tailed Averaging algorithm is a practical and widely applicable method for improving generalization.
\end{abstract}

\section{Introduction}

For the series of iterates produced by Stochastic Gradient Descent (SGD) \citep{robbins1985stochastic} to converge to a local minimum point of the training loss, the learning rate must be annealed to zero.
Polyak averaging \citep{polyak1992acceleration,ruppert1988efficient} improves on SGD and achieves a statistically optimal convergence rate by averaging all iterates to produce the final solution.
Tail or suffix averaging \citep{jain2018parallelizing,rakhlin2011making} takes this further and improves the non-asymptotic behaviour by dropping a number of leading iterates from the average, speeding up the decay of the effect of the initial state while allowing the learning rate to stay constant.
Both of these properties are advantageous in practice, where a finite number of optimization steps are taken, and because large learning rates may bias optimization towards flatter and wider minima, which improves generalization \citep{hochreiter1997flat,keskar2016large}.
Focussing on large learning rates and generalization, \citet{izmailov2018averaging} propose Stochastic Weight Averaging (SWA), which takes the same form as Tail Averaging but is motivated from an ensembling point of view.

Tail Averaging starts after a given number of optimization steps.
Setting this hyperparameter to minimize the training loss already poses some difficulties, which only become more pronounced and numerous in the context of generalization, our primary focus in this work.
\begin{itemize}
\item Triggering averaging too early is inefficient as the average must grow long to forget early weights.
\item Triggering averaging too late is inefficient as it does not use valuable information.
\item Tuning dependent hyperparameters becomes harder.
\item Early stopping is unreliable due to learning curves having a sudden drop at the onset of averaging.
\end{itemize}

Motivated by these problems, we propose the Two-Tailed Averaging algorithm with the following features:
\begin{itemize}
\item \emph{Anytime}: An estimate of the optimal tail is available at all optimization steps.
\item \emph{Adaptive}: It has no hyperparameters.
The number of weights averaged (the \emph{length} of the tail) is determined adaptively based on the evolution of generalization performance.
\item \emph{Optimal once in a while}: The tail length achieves near optimality regularly.
\end{itemize}

The algorithm is very easy to implement.
Its principal cost is the storage for a second running average, and it also performs more evaluations of generalization performance (e.g.\ the validation loss).
The main idea, sketched in \Cref{fig:schematic}, is to maintain two running averages of optimization iterates: a \emph{short} and a \emph{long} one, with the long average being our estimate of the optimal weights.

\begin{figure*}
\centering
\begin{tikzpicture}[xscale=1.8,yscale=0.7]
  \draw [{|[width=1mm]}-Circle,thick](-0.045,-0.077) -- (1,1) node[midway,above] {S=L};
  \draw [densely dotted, thick](0.955,0.795) -- (0.955,0.02);
  \draw [-{Circle[open]},thick](1,1) -- (2,2) node[midway,above] {L};
  \draw [densely dotted, thick](1.955,1.795) -- (1.955,0.02);
  \draw [{|[width=1mm]}-Circle,thick](0.955,-0.077) -- (2,1) node[midway,above] {S};
  \draw [-{Circle[open]},thick](2,1) -- (4,3) node[midway,above] {L};
  \draw [densely dotted, thick](3.955,2.795) -- (3.955,0.02);
  \draw [{|[width=1mm]}-Circle,thick](1.955,-0.077) -- (4,2) node[midway,above] {S};
  \draw [-{Circle[open]},thick](4,2) -- (7,5) node[midway,above] {L};
  \draw [densely dotted, thick](6.955,4.795) -- (6.955,0.02);
  \draw [{|[width=1mm]}-Circle,thick](3.955,-0.077) -- (7,3) node[midway,above] {S};
\end{tikzpicture}
\caption[Example evolution of the two running averages of weights over optimization steps.]{Example evolution of the two running averages of weights over optimization steps.
Line segments indicate running averages whose length (the number of optimization iterates averaged) is represented on the y axis.
The two averages start out the same, but after the first evaluation, there is always one short (S) and one long (L) average, where the long one has more iterates averaged and a better loss.
When the loss with the short one is not worse than with the long average, the long one is reset, and the short average becomes the long one.
These switch points are marked by dotted lines.
In any interval labeled with L, there is at least one point where the length of the long average is near optimal.}
\label{fig:schematic}
\end{figure*}
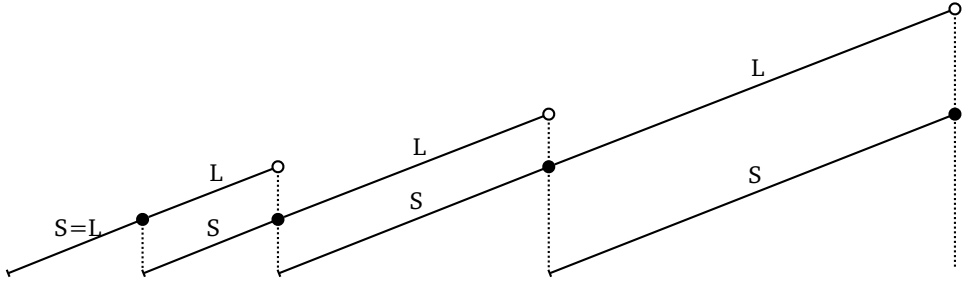

\section{Related Works}

\subsection{Averaging in Pure Optimization}

Polyak averaging as originally proposed \citep{ruppert1988efficient,polyak1992acceleration} computes the equally weighted average
$\bar{\theta}_t = \frac{1}{t+1}\sum_{i=0}^{t}\theta_i$
of all iterates $\theta_i$ from the optimizer up to the current time step $t$.
The convergence rate of $\bar{\theta}_t$ was analyzed in the convex case with an appropriately decaying learning rate.
Beyond this strictest interpretation, Polyak (or Polyak--Ruppert) averaging may refer to using $\bar{\theta}_t$ without the convexity assumption, without a decaying learning rate, or with another optimizer such as Adam \citep{kingma2014adam}.

In practice, where finite budget considerations override the asymptotic optimality guarantees offered by theory, Polyak averaging may refer to an exponential moving average (EMA) of the form
\begin{equation}
\label{eq:ema}
\begin{aligned}
\bar{\theta}_0 &= \theta_0,\\
\bar{\theta}_t &= (1-\beta_t)\theta_t + \beta_t\bar{\theta}_{t-1} & (t \geq 1),
\end{aligned}
\end{equation}
where $\beta_t < 1$ may be a constant near $1$ or it may be scheduled as in \citet{martens2020new}.
The idea here is to improve the rate of decay of the effect of the initial error by downweighting early iterates.

Tail Averaging (TA) \citep{jain2018parallelizing}, also known as Suffix Averaging \citep{rakhlin2011making}, considers a finite optimization budget of $n$ steps with a constant learning rate.
At the cost of introducing a hyperparameter $s$ to control the start of averaging, it improves the rate of decay of the effect of the initial error while obtaining near-minimax rates on the variance.
Tail Averaging is defined as
\begin{equation}
\label{eq:tail-averaging}
\begin{aligned}
\bar{\theta}_t &= \theta_t & \qquad (t < s),\\
\bar{\theta}_t &= \frac{1}{t-s+1} \sum_{i=s}^t \theta_i & \qquad (t \geq s).
\end{aligned}
\end{equation}
Alternatively, the number of iterates to average may change in proportion to the current time step:
\begin{equation*}
\begin{aligned}
\bar{\theta}_t &= \frac{1}{\ceil{ct}} \sum_{i=t+1 - \ceil{ct}}^t \theta_i & \qquad (t \geq s),
\end{aligned}
\end{equation*}
where $c \in (0,1)$ is a hyperparameter.
\citet{roux2019anytime} discusses how to approximate averages of this form without excessive storage needs but do not consider how to automatically adjust the length.

All in all, we have discussed a few representative averaging methods intended for pure optimization but often repurposed for improving generalization by tuning their hyperparameters.
Although there are interesting developments in this area \citep{shamir2013stochastic,lacoste2012simpler}, we now move on to the main focus of this work, averaging for improving generalization.

\subsection{Averaging for Generalization}

In work parallel to Tail Averaging, \citet{izmailov2018averaging} propose Stochastic Weight Averaging (SWA), an additional stage of optimization with a constant or cyclical learning rate, which computes an equally weighted average of iterates.
SWA can be motivated heuristically in the following way:
with the high learning rate, it seeks out wider and flatter basins in the training loss surface to improve generalization, but the high learning rate also prevents it from reaching the bottom of the basin, so the weights bounce around it, thus taking their average should land closer to the minimum point.
The SWA algorithm is almost identical to Tail Averaging (except for a possibly cyclical learning rate and a periodic subsampling of iterates), but it is motivated from the angle of ensembling and generalization not of optimization.

If our goal is to improve generalization, the decision of when to start averaging the weights should depend on generalization performance.
Indeed, \cite{merity2017regularizing} propose an algorithm much like SWA, where averaging is triggered when the validation loss does not improve for a fixed number of optimization steps, which trades one hyperparameter for another and is sensitive to noise in the evaluation of the generalization loss.
In other related work, \citet{guo2022stochastic} investigate the repeated application of SWA.
Their method is not informed by the validation loss and requires the schedule of multiple SWA stages to be specified.
Finally, taking the exponential moving average of iterates is also sensitive to its hyperparameter, the decay rate.

In summary, existing averaging methods for generalization that behave well in practice all have one or more hyperparameters to govern the weighting of early iterates.
Tuning these hyperparameters can be costly, particularly in the presence of other hyperparameters and when training runs take a long time.
Furthermore, even with their hyperparameters, these methods are not flexible enough to estimate the optimal average at multiple optimization steps in general.
We address these issues in the present work.
The rest of this chapter is structured as follows.
In \Cref{sec:tta-problem-statement}, we formally define the problem to solve.
In \Cref{sec:tta-algorithm}, we provide a description of the algorithm, whose properties are analyzed in \Cref{sec:tta-analysis}.
We verify our analysis experimentally in \Cref{sec:tta-experiments} and discuss the validity of our assumptions in \Cref{sec:tta-failed-assumptions}.

\section{Problem Statement}
\label{sec:tta-problem-statement}

Let $\Theta$ be the parameter (or weight) space, $\theta_t \in \Theta$ $(t \in \natnumzero)$ a sequence of iterates produced by stochastic optimization with $\theta_0$ being the initial value, and $f \colon \Theta \to \Rb$ the generalization loss function.
We may choose the generalization loss to simply be the validation loss, or it may measure performance on a down-stream task.
We assume the generalization loss is evaluated periodically, every $E \in \natnum$ optimization steps, and it is at these points $n \in [0, E, 2E, 3E, \dots]$ where we would like to know how many of the most recent iterates to average to minimize it.
Here and in the following, $t$ and $n$ (with or without subscripts) are assumed to be from $\natnumzero$ and $[0, E, 2E, 3E, \dots]$, respectively, and \emph{loss} always refers to $f$.
Denoting the average of most recent $\Delta$ iterates up to time step $t$ with $\avg(t, \Delta) = \frac{1}{\Delta} \sum_{i=t+1-\Delta}^{t} \theta_i$, we define the optimal averaging length as
\begin{align*}
\OL(t) = \argmin_{\Delta \in [1, \dots, t]} f(\avg(t, \Delta)).
\end{align*}
Our task is to approximate $\OL(n)$ and $\avg(n, \OL(n))$ at all evaluation steps $n$ during optimization.

The trivial algorithm to find $\OL(n)$, which saves all $\theta_i$ and performs a search over $\Delta \in [1, \dots, n]$ to minimize $f(\avg(n, \Delta))$, has prohibitive storage and evaluation cost, proportional to $n$.
Even assuming that $f$ improves monotonically in $\Delta$ up to its optimum, the cost is still proportional to $\OL(n)$.
Our proposed algorithm approximates $\OL(n)$ and $\avg(n, \OL(n))$ with a constant cost.

\section{The Algorithm}
\label{sec:tta-algorithm}

\begin{algorithm}[t]
  \caption[The core Two-Tailed Averaging algorithm, without extensions.]{\small The core Two-Tailed Averaging algorithm.
It has 2 running averages, a short one $\theta^S$ and a long one $\theta^L$ with $S \leq L$ number of optimization iterates averaged.
If the loss with $\theta^S$ becomes lower or equal to the loss with $\theta^L$, then we empty the long average, which becomes the short one.}
  \label{alg:two-tailed-averaging-core}
  \begin{algorithmic}[1]
    \Require{generalization loss function $f$,
      optimization iterates $\theta_t$, evaluation period $E$}
    \Let{$S,\theta^S, L, \theta^L$} {$0, \mathbf{0}, 0, \mathbf{0}$}
    \MyComment{The first evaluation will cause a switch.}
    \State
    \Procedure{$add\_weights$}{$\theta$}
      \Let{$S, \theta^S$}{$S+1,\theta^S + (\theta-\theta^S)/(S+1)$}
      \MyComment{Add $\theta$ to the short average}
      \Let{$L, \theta^L$}{$L+1,\theta^L + (\theta-\theta^L)/(L+1)$}
      \MyComment{Add $\theta$ to the long average}
    \EndProcedure
    \State
    \Procedure{$switch$}{}
      \MyComment{Reset the long average}
      \Let{$S, L,\theta^L$}{$0,S,\theta^S$}
      \MyComment{Must switch short and long to maintain $S \leq L$}
    \EndProcedure
    \State
    \Function{$evaluate\_and\_adapt$}{$\theta, f$}
      \Let{$F^S,F^L$}{$f(\theta^S),f(\theta^L)$}
      \If{$F^S \leq F^L$}
        \MyComment{Is the short average better?}
        \Let{$F^L$}{$F^S$}
        \State $switch()$
        \label{algo:line:switch}
      \EndIf
      \State \Return{$F^L,\theta^L,L$}
    \EndFunction
    \State
    \For{$t \gets 1, 2, \dots$}
      \MyComment{Training loop}
      \State $add\_weights(\theta_t)$
      \MyComment{$\theta_t$ comes from the ongoing optimization}
      \label{algo:line:beforeswitch}
      \If{$t \bmod E = 0$}
        \MyComment{Evaluate $f$ every $E$ iterates}
        \Let{$f_{\bar{\theta}}, \bar{\theta}, len$}{$evaluate\_and\_adapt(\theta_t, f)$}
        \State{$report(t, f_{\bar{\theta}}, \bar{\theta}, len)$}
      \EndIf
    \EndFor
  \end{algorithmic}
\end{algorithm}

\Cref{alg:two-tailed-averaging-core} specifies the core of Two-Tailed Averaging (\tta{}) in pseudocode, which works as follows.
The training loop iterates over weights $\theta_t$ produced by a stochastic optimizer, incorporating them into the short and long running averages $\theta^S$, $\theta^L$ with lengths $S$ and $L$.
Then, every $E$ steps, the loss is evaluated with the short average $\theta^S$ and with the long average $\theta^L$, giving $F^S$ and $F^L$.
If $F^S$ is at least as good as $F^L$, then we switch: the long average is reset,
and since that makes it the shorter of the two averages, we must switch their labels.
In other words, on a switch, the long average continues from the current short average and the short average is restarted (see \Cref{fig:schematic}).
For time step $t$, the estimate of the optimal averaging length is $L$, and $\theta^L$ is the corresponding average.

\begin{algorithm}[t]
  \caption[Extensions to Two-Tailed Averaging.]{\small Extensions to Two-Tailed Averaging.
This is the version recommended for use in practice.
The parts unchanged from \Cref{alg:two-tailed-averaging-core} are grayed out.
There are two extensions.
First, the short and long averages are reset if they are stagnating (i.e.\  they have not improved for a few evaluations).
This reset heuristic makes the algorithm quicker to adapt when \Cref{monotone-opt} is violated (see \Cref{sec:tta-failed-assumptions}).
Second, we defer to non-averaged weights very early in training.}
  \label{alg:two-tailed-averaging-with-extensions}
  \begin{algorithmic}[1]
    \color{gray!80}
    \Function{$evaluate\_and\_adapt$}{$\theta, f$}
      \Let{\textcolor{black}{$F^1,$}$F^S,F^L$}
          {\textcolor{black}{$f(\theta),$}$f(\theta^S),f(\theta^L)$}
      \If{$F^S \leq F^L$ \textcolor{black}{or $F^L$ is stagnating}}
        \Let{$F^L$}{$F^S$}
        \State $switch()$
      \color{black}
      \ElsIf{$F^S$ is stagnating}
        \Let{$S$}{$0$}
      \EndIf
      \If{$L > 1$ and $F^1 \leq F^L$}
        \MyComment{Use the non-averaged weights if better}
        \If{$L=E$}
          \MyComment{If they have always been better so far,}
          \Let{$S,L$}{$0,0$}
          \MyComment{\dots then reinitialize}
        \EndIf
        \State \Return{$F^1,\theta,1$}
      \EndIf
      \color{gray!80}
      \State \Return{$F^L,\theta^L,L$}
    \EndFunction
  \end{algorithmic}
\end{algorithm}

In \Cref{alg:two-tailed-averaging-with-extensions}, we present two heuristic extensions to the core algorithm.
First, the long and short averages are reset if they have not improved for a few evaluations.
This reset heuristic is intended to handle cases where the averages become too long, perhaps due to optimization escaping from one basin of attraction to a better one or due to the loss surface changing in a non-stationary environment.
Second, we defer to the non-averaged weights very early in training, where $f(\theta_t)$ is still improving rapidly enough that averaging the minimum $E$ iterates is worse than not averaging at all.


\section{Analysis of the Algorithm}
\label{sec:tta-analysis}

Our analysis hinges on simplifying assumptions, which follow from, for example, a monotonically decreasing loss and averaging producing diminishing returns as the length increases.
They represent idealized circumstances; we discuss their validity and failures in \Cref{sec:tta-failed-assumptions}.

\begin{assumption}\label{easy-opt}
For all $n$, as a function of $\Delta \in [0, E, 2E, \dots, \OL_E(n)]$, $f(\avg(n,\Delta))$ is monotonically decreasing, where $\OL_E(n)=\floor{\OL(n)/E}E$.
That is, for any given evaluation step $n$, averaging more iterates from the past monotonically improves $f$ until about the optimum length.
\end{assumption}

\begin{assumption}\label{easy-opt2}
For all $n$ and $n_+$, such that $n_+ \geq \OL^E(n)$, $f(\avg(n,\OL^E(n))) \leq f(\avg(n,n_+))$, where $\OL^E(n)=\ceil{\OL(n)/E}E$.
That is, averaging slightly more than optimal is better than averaging a lot more.
\end{assumption}

\begin{assumption}\label{slow-opt}
$\forall n \colon \exists n_s \colon \OL(n+n_s) - \OL(n) < n_s$, that is, the optimal average forgets over a sufficiently long interval.
\end{assumption}

\begin{assumption}\label{monotone-opt}
$\OL(n) \leq \OL(n+E)$, that is, the optimal number of weights to average is monotonically increasing from one evaluation to the next.
\end{assumption}

Let $S(t)$, $\theta^S(t)$, $L(t)$, and $\theta^L(t)$ stand for the values of variables $S$, $\theta^S$, $L$, and $\theta^L$ in \Cref{alg:two-tailed-averaging-core}, respectively, after $t$ times through the loop.
Similarly, let $S'\!(t)$, $\theta^{S'\!\!}(t)$, $L'\!(t)$, and $\theta^{L'\!\!}(t)$ stand for the values of the same variables at the same iteration but after \Cref{algo:line:beforeswitch} (i.e.\ before the possible switch of the short and long averages).
Furthermore, we introduce the shorthands $f^X(t) = f(\avg(t,X(t)))$ for $X \in \{S,S',L,L',\OL\}$ with $f^X(t) = +\infty$ if $X(t)=0$.

\begin{definition}[Switch point]
We say that $n$ is a switch point if at $t=n$ \Cref{algo:line:switch} is executed, that is, when the short average becomes at least as good as the long average with respect to the loss, and consequently the long average is reset.
We denote the most recent switch point before iteration $t$ with $\SP(t)$, where $\SP(t) < t$.
If there is no such switch point, then $\SP(t)=-1$.
\end{definition}

\Cref{monotone-opt} states that the optimal averaging length monotonically increases, so to simplify the analysis, without loss of generality, we assume throughout that the raw loss $F^1$ has already been eclipsed by $F^L$ at the first evaluation.
We also assume that the reset heuristic cannot trigger.
In effect, we ignore the extensions in \Cref{alg:two-tailed-averaging-with-extensions} and analyze the core logic in \Cref{alg:two-tailed-averaging-core}.
For the proofs of the following propositions, see \Cref{sec:proofs-for-analysis}.

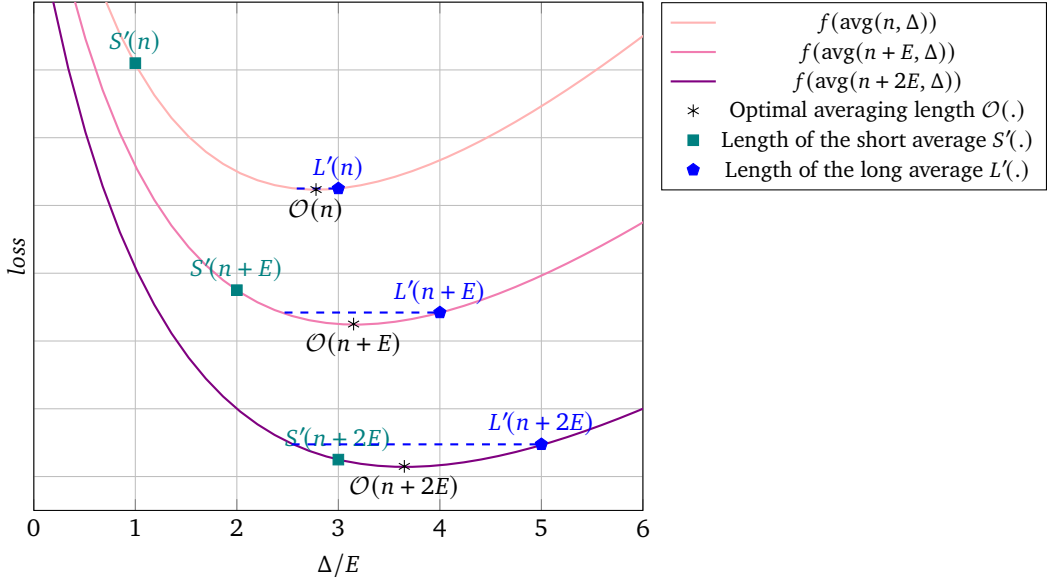
\begin{figure*}
\centering
\begin{tikzpicture}
\begin{axis}[xmin=0,xmax=6,samples=60,grid=major,ymin=-2.3,ymax=-0.8,
    xlabel=$\Delta/E$, ylabel=$loss$, ymajorticks=false,
    height=0.5*\textwidth,
    width=0.58*\textwidth,
    legend pos=outer north east,
    legend style={nodes={scale=0.9, transform shape}}
    ]
\addplot [color=red!30!white,domain=0:10,style=thick]
    {-4+4/(1+0.5*x)+0.35*x};
\addlegendentry{$f(\avg(n,\Delta))$}
\addplot [color=magenta!65!white,domain=0:10,style=thick]
    {-4+4/(1+0.5*x)+0.30*x-0.25};
\addlegendentry{$f(\avg(n+E,\Delta))$}
\addplot [color=violet!100!white,domain=0:10,style=thick]
    {-4+4/(1+0.5*x)+0.25*x-0.5};
\addlegendentry{$f(\avg(n+2E,\Delta))$}

\addplot[only marks, color=black, mark=asterisk, mark options={scale=1.2}]
  coordinates {(2.78,-1.353)}
  node[below,pos=1] {$\OL(n)$};
\addlegendentry{Optimal averaging length $\OL(.)$}
\addplot[color=black, mark=asterisk, mark options={scale=1.2}, forget plot]
  coordinates {(3.15,-1.75)}
  node[below,pos=1] {$\OL(n+E)$};
\addplot[color=black, mark=asterisk, mark options={scale=1.2}, forget plot]
  coordinates {(3.65,-2.17)}
  node[below,pos=1] {$\OL(n+2E)$};

\addplot[only marks, color=teal, mark=square*]
  coordinates {(1.0,-0.98)}
  node[above,pos=1] {$S'\!(n)$};
\addlegendentry{Length of the short average $S'\!(.)$}
\addplot[color=teal, mark=square*, forget plot]
  coordinates {(2.0,-1.65)}
  node[above,pos=1] {$S'\!(n+E)$};
\addplot[color=teal, mark=square*, forget plot]
  coordinates {(3.0,-2.15)}
  node[above,pos=1] {$S'\!(n+2E)$};

\addplot[only marks, color=blue, mark=pentagon*, mark options={scale=1.2}]
  coordinates {(3,-1.35)}
  node[above,pos=1] {$L'\!(n)$};
\addlegendentry{Length of the long average $L'\!(.)$}
\addplot[color=blue,domain=2.59:3, thick, dashed, forget plot] {-1.35};
\addplot[color=blue, mark=pentagon*, mark options={scale=1.2}, forget plot]
  coordinates {(4.0,-1.716)}
  node[above,pos=1] {$L'\!(n+E)$};
\addplot[color=blue,domain=2.47:4,thick, dashed, forget plot] {-1.716};
\addplot[color=blue, mark=pentagon*, mark options={scale=1.2}]
  coordinates {(5.0,-2.105)}
  node[above,pos=1] {$L'\!(n+2E)$};
\addplot[color=blue,domain=2.54:5,thick, dashed, forget plot] {-2.105};

\end{axis}
\end{tikzpicture}
\caption[Idealized illustration of switching.]{Idealized illustration of switching.
The three curves show the loss as a function of averaging length at three subsequent evaluations (at optimization steps $n$, $n+E$, and $n+2E$).
The raw loss keeps improving, hence later evaluations have lower loss curves. The optimal averaging length increases, $\OL(n) \leq \OL(n+E) \leq \OL(n+2E)$, as per \Cref{monotone-opt}.
At $t=n+2E$, where the loss of the short average dips below the loss of the long average, the long average is reset, and the short average becomes the long average, so we have $S(n+2E)=0$ and $L(n+2E)=S'\!(n+2E)$.}
\label{fig:switching}
\end{figure*}

\begin{restatable}[Bounds for the averaging lengths]{proposition}{propbounds}
\label{prop:bounds}
The lengths of the short and long averages are bounded as $S(n) < \OL(n)$ and $L(n) < 2\OL(n) + E$.
\end{restatable}
\ifdef{\groundskip}{\vspace*{-\groundskip}}{}

\begin{restatable}[Infinite number of switch points]{proposition}{propinfinite}
\label{prop:infinite-switch-points}
Switch points keep coming, that is, $\forall n \colon \exists n_s \geq n \colon \SP(n_s) \neq -1$.
\end{restatable}
\ifdef{\groundskip}{\vspace*{-\groundskip}}{}

\begin{restatable}[Once-in-a-while optimality]{proposition}{proponce}
\label{prop:once-in-a-while-optimality}
Between any two subsequent switch points $n_1$ and $n_2$, the long average is nearly optimal at least once.
Formally, $\exists n \in  [n_1, n_2-E] \colon L(n) = \OL^E(n) \lor L(n) = \OL_E(n)$.
\end{restatable}

In short, we have shown that the long average is at most twice as long as optimal, there are infinitely many switch points, and between any two switch points the long average is approximately optimal at least once.
Our results are in terms of lengths of averages, and relating the actual loss with the long average $f^L$ to the loss with the optimal length $f^{\OL}$ would be desirable.
Here, we informally point out that, all things being equal, the worse $f^L$ gets relative to $f^{\OL}$, the quicker $f^S$ is to catch up with $f^L$, making long periods of highly suboptimal solution less likely.
Formalizing this notion requires making further assumptions about the loss-vs-averaging-length function (of the kind plotted in \Cref{fig:switching}) and would make analysis considerably more cumbersome.

\subsection{When Assumptions Fail}
\label{sec:tta-failed-assumptions}

To augment the theoretical analysis, which is based on idealized assumptions, we make the following observations.
The strongest assumption by far is \Cref{easy-opt}.
It says that increasing the averaging length monotonically improves $f$ until the optimum.
Since stochastic optimization produces noisy iterates, this does not hold exactly in practice.
However, the length of the shortest average is one evaluation period, and its variance is inversely proportional to $E$.
Thus, the likelihood of noise posing a problem can be very small.
In terms of the loss, the algorithm is fairly robust to when the assumption holds only approximately because small deviations of $f(\avg(n,\Delta))$ from monotonicity can change switch times only when $f^S$ and $f^L$ are close.

\Cref{easy-opt2} says that averaging slightly more iterates than optimal (i.e.\ rounded up to the evaluation period) is better than averaging a lot more.
This is a weak assumption due to subsequent iterates being highly correlated.
If it is violated sporadically, the algorithm can fail to detect when the short average becomes longer than optimal, which delays the switch.

\Cref{slow-opt} failing means that the optimal average incorporates all new iterates without ever dropping old ones.
In this case, the short average, which is always shorter than optimal, will be a constant number of iterates behind and its loss will converge to the loss of the optimal average.
If the long average is shorter than optimal, then the same argument applies to it.
If the long average is longer than optimal, then eventually a switch will happen.
In either case, the loss of the long average converges to that of the optimal average.

Regarding \Cref{monotone-opt}, $\OL(n) \leq \OL(n+E)$ can fail if the improvement of the raw loss accelerates, but that is a rather uncommon and temporary occurrence.
It may also fail if the raw loss has started to worsen due to overfitting or optimization has escaped from one basin to the next and the average is slowly climbing the ridge separating them or when the loss landscape changes during learning in a non-stationary environment.
With the exception of accelerating improvement, these are likely to be caught by the reset heuristic, wherein the long average is reset if its loss does not improve for a few evaluations (see \Cref{alg:two-tailed-averaging-with-extensions}).
The reset heuristic can trigger when it should not, i.e.\ when \Cref{monotone-opt} holds.
Such a spurious reset makes the estimate of the long average worse either directly or indirectly by delaying the next switch.
Either way, without further violations of this assumption, the algorithm recovers by the next switch.
Note that \tta{} cannot in general correct overfitting, although the reset heuristic may help in the unlikely event that overfitting is transitory.

All in all, we can expect the algorithm to display some degree of robustness to minor violations of the assumptions.
In practice, we recommend choosing a reasonably large $E$ to reduce the noise originating from the stochasticity of optimization.

\subsection{A Note on Pure Optimization}
\label{sec:tta-pure-optimization}

Applying \tta{} to pure optimization is unlikely to bring about practical benefits because of the evaluation cost.
For example, if $f$ computes the loss over the entire training set and evaluation is performed every epoch, then the cost of optimization is effectively doubled.
Furthermore, as we have pointed out above, \Cref{easy-opt} does not hold exactly with stochastic optimization: the short average can get lucky and become better than the long one (causing a switch) but then quickly succumb to variance and become worse as new iterates are added to it.
Thus, the averaged weights of \Cref{alg:two-tailed-averaging-core,alg:two-tailed-averaging-with-extensions} do not converge in the strict sense, although this point is somewhat moot because --~due to the mismatch between the true and the training losses~-- convergence in the training loss is almost never desirable when optimizing for generalization.
Nevertheless, it is instructive to consider how the algorithm behaves when $f$ is the training loss as the limit of the common case where $f$ is the validation loss, both the training and the validation sets consist of i.i.d. samples from the same distribution, and their sizes tend to infinity.
Focussing on the setting where convergence results are available for Polyak and Tail Averaging, we show in \Cref{sec:properties-of-tta-in-pure-optimization} that \tta{} converges in probability to the optimum in ordinary least squares regression.
We also show that the losses of the short averages at switch points monotonically decrease if $f$ is convex.

\section{Experiments}
\label{sec:tta-experiments}

\begin{figure}[t]
  \centering
  \includegraphics[width=0.9\textwidth]{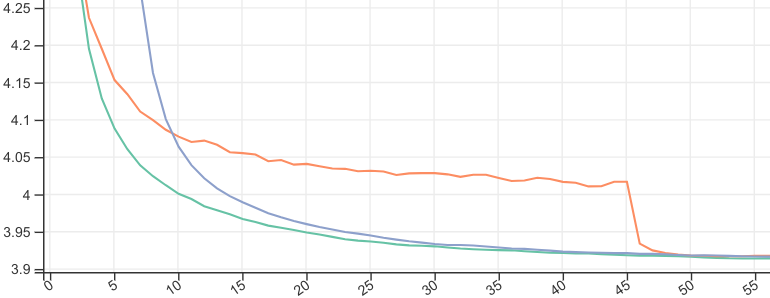}
  \caption[Validation losses with Two-Tailed Averaging and baselines.]{Validation loss with Two-Tailed Averaging (\tta{}, green), Tail Averaging (TA, orange), and an exponential moving average (EMA, blue) of weights on language modelling on Penn Treebank.
For both TA and EMA, their hyperparameters (the start time and the decay rate) were tuned to minimize the final loss, so it is not a surprise that all three have similar optima.
\tta{} has no hyperparameters and produces much better early solutions.
These two factors make tuning easier and early stopping much more reliable.
Additionally, the noise in the raw loss is effectively smoothed out.
Note that while the \tta{} loss decreases monotonically, gentler and steeper slopes are manifest before and after switch points, respectively.}
  \label{fig:swa-vs-oooswa}
\end{figure}

\begin{figure}[t]
  \centering
  \begin{tikzpicture}[scale=\ifdef{\groundskip}{0.58}{1.0}]
    \node [] at (0,0){\includegraphics[width=0.9\textwidth]{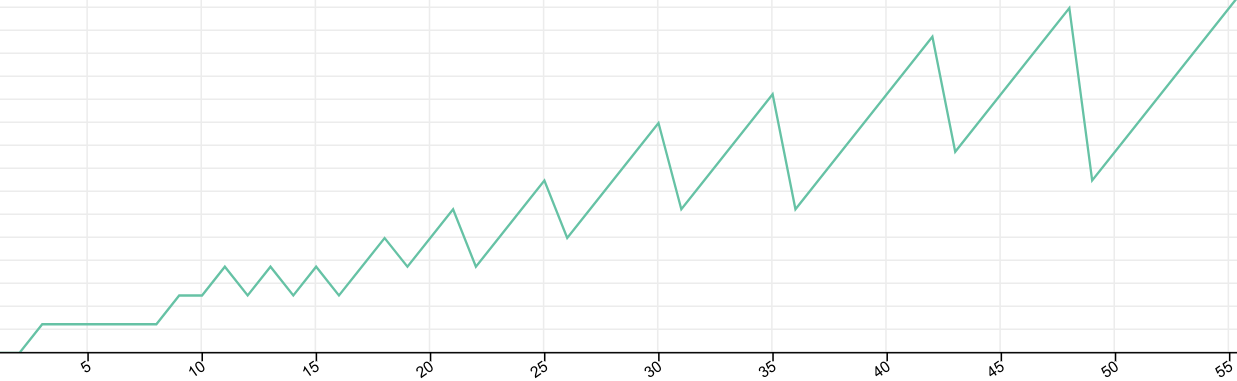}};
    \draw [{|[width=1mm]}-Circle,thick](-0.6,-1.66) -- (0.70,-0.105) node[midway,above] {S};
    \draw [-{Circle[open]},thick](0.63,-0.198) -- (1.85,1.33) node[midway,above] {L};
  \end{tikzpicture}
  \caption[Length of the long average ($L$) vs number of evaluations of \tta{}.]{Length of the long average ($L$) vs number of evaluations of \tta{} in \Cref{fig:swa-vs-oooswa}.
Note how the heights of both peaks and valleys increase almost monotonically.
Also, the correspondence between this figure and the schematic in \Cref{fig:schematic} is illustrated on one of the $S$ and $L$ intervals.}
  \label{fig:length}
\end{figure}

Tail Averaging or Stochastic Weight Averaging have been shown previously to be beneficial not only in theory and on simulated data \citep{jain2018parallelizing} but also in language modelling \citep{merity2017regularizing,melis2019mogrifier} and image classification \citep{izmailov2018averaging} experiments.
Hence, in this work, we restrict our attention to experiments in a single domain to corroborate the analysis in \Cref{sec:tta-analysis}.
Our goals are to \emph{i)} verify that \tta{} is on par with well-tuned TA and EMA, \emph{ii)} explore the effect of basing the switching logic on the training instead of the validation loss, \emph{iii)} demonstrate robustness to the choice of evaluation period, \emph{iv)} and check whether the assumptions in \Cref{sec:tta-analysis} hold in practice.

In particular, we trained a recurrent language model with several hyperparameters on Penn Treebank \citep{mikolov2010recurrent} using the Rectified Adam optimizer \citep{liu2019variance}, evaluating every 1000 optimization steps.
The hyperparameters were tuned separately for \tta{}, TA \cref{eq:tail-averaging}, and EMA \cref{eq:ema}.
\Cref{fig:swa-vs-oooswa} shows that the final validation losses with all methods are very close, but early losses with \tta{} are much better.
This is expected because TA and EMA are not flexible enough to produce optimal averaging lengths at multiple points along the learning curve despite having an extra hyperparameter.
Conversely, \tta{} has at least one nearly optimal solution between any two subsequent peaks (i.e.\ switch points) in \Cref{fig:length} despite having no hyperparameters.

We also tried the version of the algorithm where the switching logic was based on comparison of the training losses of the short and long averages instead of the validation losses, but the true validation losses were reported.
On this particular language modelling task, the best validation loss with the modified algorithm worsened moderately (3.93 vs 3.92) and was well below the raw validation loss (4.02).
Results on the test set exhibited the same gap.
Since the modified \tta{} was minimizing the training loss, the smoothness of the reported validation losses observed in \Cref{fig:swa-vs-oooswa} were lost in the process.
Similar results were obtained by scheduling a learning rate drop without averaging.

To explore the effect of the evaluation period $E$, we tuned models with four times larger and four times smaller $E$ than in our previously discussed experiments.
As expected, the best final results were very close to each other, with shorter periods having an advantage early in training as the raw loss $F^1$ was more quickly eclipsed by $F^L$.\looseness=-1

In addition, we found that the assumptions made in \Cref{sec:tta-analysis} held rather well: $F^L$ in \Cref{fig:swa-vs-oooswa}, and the averaging lengths tended to change monotonically (see heights of peaks and valleys in \Cref{fig:length}), making our length-based theoretical results more closely linked to the actual loss.
When that was not the case, we found that the raw loss $F^1$ had started to worsen due to overfitting or, much more rarely, optimization had entered a new basin, violating \Cref{monotone-opt}.
\Cref{fig:overfitting} and \Cref{fig:new-basin} demonstrate the reset heuristic being triggered in these cases.
Finally, TA and \tta{} having almost identical final validation losses weakly supports our assumptions, although a conclusive demonstration would need to plot the results obtained with TA tuned separately for each evaluation.

\section{Conclusions}

\begin{figure}[t]
  \centering
  \includegraphics[width=0.9\textwidth]{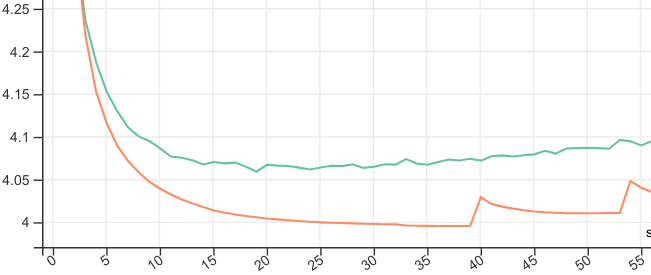}
  \caption[Raw and \tta{} validation loss with overfitting.]{Raw (green) and \tta{} (orange) validation loss with overfitting.
The averaged loss bottoms out due to overfitting.
Thus when the averages are reset (twice), the validation loss does not recover.
Although the losses reported after the reset are suboptimal, it does not really matter as a better loss was reported already.}
  \label{fig:overfitting}
\end{figure}

\begin{figure}[t]
  \begin{subfigure}{0.52\textwidth}
    \includegraphics[width=\linewidth,height=2.75cm,clip]{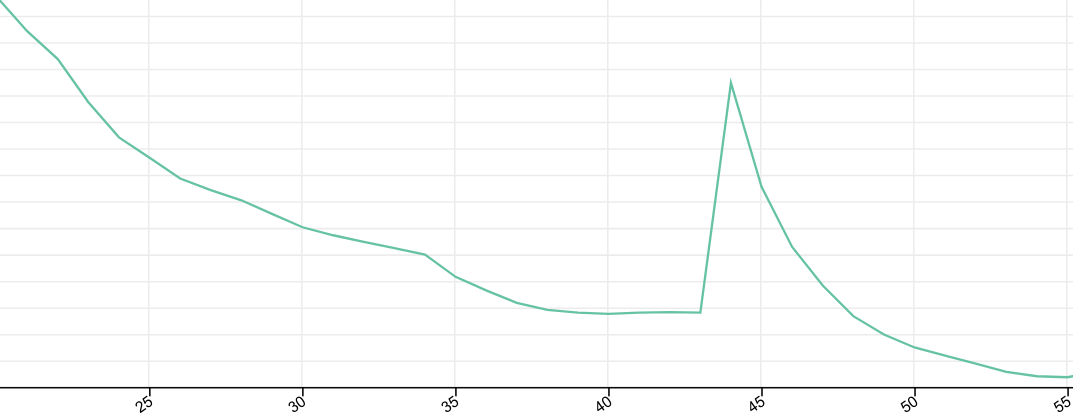}
    \caption{Validation loss $F^L$.}
  \end{subfigure}
  \hfill
  \begin{subfigure}{0.45\textwidth}
    \includegraphics[width=\linewidth,height=2.75cm,clip]
                    {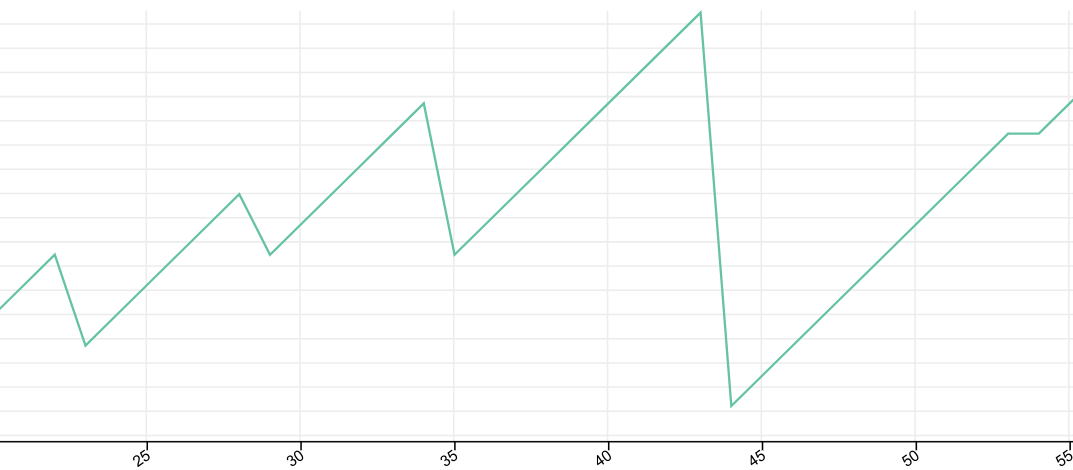}
    \caption{Length $L$.}
  \end{subfigure}
  \caption[Example of optimization entering a new basin.]{Example of optimization entering a new basin.
At $t=40E$ the validation loss with the long average bottoms out and starts to slightly worsen.
This is detected at $t=43E$, and both averages are reset.
With the averages now way too short, the loss spikes but then recovers.}
  \label{fig:new-basin}
\end{figure}

Tail averaging improves on Polyak averaging's non-asymptotic behaviour by excluding a number of leading iterates of stochastic optimization from its calculations.
In practice, with a finite number of optimization steps and a learning rate that cannot be annealed to zero, Tail Averaging can get much closer to a local minimum point of the training loss than either the individual iterates or the Polyak average.
However, the number of leading iterates to ignore is an important hyperparameter, and starting averaging too early or too late leads to inefficient use of resources or suboptimal solutions.
Our work focussed on improving generalization, which makes setting this hyperparameter even more difficult, especially in the presence of other hyperparameters and overfitting.
Furthermore, before averaging starts, the loss is only weakly informative of the final performance, which makes early stopping unreliable.
To alleviate these problems, we propose an anytime variant of Tail Averaging intended for improving generalization not pure optimization that has no hyperparameters and approximates the optimal tail at all optimization steps.
Our algorithm is based on two running averages with adaptive lengths bounded in terms of the optimal tail length, one of which achieves approximate optimality with some regularity.

In summary, we presented a variant of Tail Averaging and Stochastic Weight Averaging based on two running averages.
Compared to them, Two-Tailed Averaging requires additional storage for the second running average and relies on periodic evaluation of generalization performance.
In return, \tta{} removes a hyperparameter and provides an estimate of the optimal tail at all optimization steps.
This makes hyperparameter tuning easier and early evaluation more representative of final performance, allowing it to support early and anytime stopping better.
Owing to its simplicity, low implementation cost and adaptivity, \tta{} is a practical and widely applicable method for improving generalization.

Looking beyond the scope of this work, exploring the relationship between iterate averaging and learning rate schedules is a promising direction as existing \citep{merity2017regularizing} and our own limited experimental results indicate that dropping the learning rate and Tail Averaging perform comparably.
The properties of our algorithm are particularly compelling for continual learning: by allowing the learning rate to remain high and being able to adapt the averaging length to changing circumstances, \tta{} lets the model maintain high plasticity while reaping the benefits of averaging.

In addition, averaging weights can be viewed as a cheap approximation to averaging predictions when the averaged weights reside in a region with a suitable geometry.
The combination of averaging weights within such regions and averaging predictions over regions (each with its own weight average) could potentially achieve a better loss than weight averaging alone at much lower storage and evaluation cost than pure prediction averaging.
We leave these avenues for future work to explore.



{
  \ifdef{\groundskip}{}{\clearpage}
  \bibliography{paper}

\begin{thebibliography}{19}
\providecommand{\natexlab}[1]{#1}
\providecommand{\url}[1]{\texttt{#1}}
\expandafter\ifx\csname urlstyle\endcsname\relax
  \providecommand{\doi}[1]{doi: #1}\else
  \providecommand{\doi}{doi: \begingroup \urlstyle{rm}\Url}\fi

\bibitem[Guo et~al.(2022)Guo, Jin, and Liu]{guo2022stochastic}
Hao Guo, Jiyong Jin, and Bin Liu.
\newblock Stochastic weight averaging revisited.
\newblock \emph{arXiv preprint arXiv:2201.00519}, 2022.

\bibitem[Hochreiter and Schmidhuber(1997)]{hochreiter1997flat}
Sepp Hochreiter and J{\"u}rgen Schmidhuber.
\newblock Flat minima.
\newblock \emph{Neural computation}, 9\penalty0 (1):\penalty0 1--42, 1997.

\bibitem[Izmailov et~al.(2018)Izmailov, Podoprikhin, Garipov, Vetrov, and
  Wilson]{izmailov2018averaging}
Pavel Izmailov, Dmitrii Podoprikhin, Timur Garipov, Dmitry Vetrov, and
  Andrew~Gordon Wilson.
\newblock Averaging weights leads to wider optima and better generalization.
\newblock \emph{arXiv preprint arXiv:1803.05407}, 2018.

\bibitem[Jain et~al.(2018)Jain, Kakade, Kidambi, Netrapalli, and
  Sidford]{jain2018parallelizing}
Prateek Jain, Sham Kakade, Rahul Kidambi, Praneeth Netrapalli, and Aaron
  Sidford.
\newblock Parallelizing stochastic gradient descent for least squares
  regression: mini-batching, averaging, and model misspecification.
\newblock \emph{Journal of Machine Learning Research}, 18, 2018.

\bibitem[Keskar et~al.(2016)Keskar, Mudigere, Nocedal, Smelyanskiy, and
  Tang]{keskar2016large}
Nitish~Shirish Keskar, Dheevatsa Mudigere, Jorge Nocedal, Mikhail Smelyanskiy,
  and Ping Tak~Peter Tang.
\newblock On large-batch training for deep learning: Generalization gap and
  sharp minima.
\newblock \emph{arXiv preprint arXiv:1609.04836}, 2016.

\bibitem[Kingma and Ba(2014)]{kingma2014adam}
Diederik~P Kingma and Jimmy Ba.
\newblock Adam: A method for stochastic optimization.
\newblock \emph{arXiv preprint arXiv:1412.6980}, 2014.

\bibitem[Lacoste-Julien et~al.(2012)Lacoste-Julien, Schmidt, and
  Bach]{lacoste2012simpler}
Simon Lacoste-Julien, Mark Schmidt, and Francis Bach.
\newblock A simpler approach to obtaining an o (1/t) convergence rate for the
  projected stochastic subgradient method.
\newblock \emph{arXiv preprint arXiv:1212.2002}, 2012.

\bibitem[Liu et~al.(2019)Liu, Jiang, He, Chen, Liu, Gao, and
  Han]{liu2019variance}
Liyuan Liu, Haoming Jiang, Pengcheng He, Weizhu Chen, Xiaodong Liu, Jianfeng
  Gao, and Jiawei Han.
\newblock On the variance of the adaptive learning rate and beyond.
\newblock \emph{arXiv preprint arXiv:1908.03265}, 2019.

\bibitem[Martens(2020)]{martens2020new}
James Martens.
\newblock New insights and perspectives on the natural gradient method.
\newblock \emph{The Journal of Machine Learning Research}, 21\penalty0
  (1):\penalty0 5776--5851, 2020.

\bibitem[Melis et~al.(2019)Melis, Ko{\v{c}}isk{\`y}, and
  Blunsom]{melis2019mogrifier}
G{\'a}bor Melis, Tom{\'a}{\v{s}} Ko{\v{c}}isk{\`y}, and Phil Blunsom.
\newblock Mogrifier {LSTM}.
\newblock \emph{arXiv preprint arXiv:1909.01792}, 2019.

\bibitem[Merity et~al.(2017)Merity, Keskar, and Socher]{merity2017regularizing}
Stephen Merity, Nitish~Shirish Keskar, and Richard Socher.
\newblock Regularizing and optimizing {LSTM} language models.
\newblock \emph{arXiv preprint arXiv:1708.02182}, 2017.

\bibitem[Mikolov et~al.(2010)Mikolov, Karafi{\'a}t, Burget, Cernock{\`y}, and
  Khudanpur]{mikolov2010recurrent}
Tomas Mikolov, Martin Karafi{\'a}t, Lukas Burget, Jan Cernock{\`y}, and Sanjeev
  Khudanpur.
\newblock Recurrent neural network based language model.
\newblock In \emph{Interspeech}, volume~2, page~3, 2010.

\bibitem[Polyak and Juditsky(1992)]{polyak1992acceleration}
Boris~T Polyak and Anatoli~B Juditsky.
\newblock Acceleration of stochastic approximation by averaging.
\newblock \emph{SIAM journal on control and optimization}, 30\penalty0
  (4):\penalty0 838--855, 1992.

\bibitem[Rakhlin et~al.(2011)Rakhlin, Shamir, and Sridharan]{rakhlin2011making}
Alexander Rakhlin, Ohad Shamir, and Karthik Sridharan.
\newblock Making gradient descent optimal for strongly convex stochastic
  optimization.
\newblock \emph{arXiv preprint arXiv:1109.5647}, 2011.

\bibitem[Robbins and Monro(1985)]{robbins1985stochastic}
Herbert Robbins and Sutton Monro.
\newblock A stochastic approximation method.
\newblock In \emph{Herbert Robbins Selected Papers}, pages 102--109. Springer,
  1985.

\bibitem[Roux(2019)]{roux2019anytime}
Nicolas~Le Roux.
\newblock Anytime tail averaging, 2019.

\bibitem[Ruppert(1988)]{ruppert1988efficient}
David Ruppert.
\newblock Efficient estimations from a slowly convergent {R}obbins-{M}onro
  process.
\newblock Technical report, Cornell University Operations Research and
  Industrial Engineering, 1988.

\bibitem[Shamir and Zhang(2013)]{shamir2013stochastic}
Ohad Shamir and Tong Zhang.
\newblock Stochastic gradient descent for non-smooth optimization: Convergence
  results and optimal averaging schemes.
\newblock In \emph{International conference on machine learning}, pages 71--79.
  PMLR, 2013.

\bibitem[Yu et~al.(2020)Yu, Balasubramanian, Volgushev, and
  Erdogdu]{yu2020analysisv2}
Lu~Yu, Krishnakumar Balasubramanian, Stanislav Volgushev, and Murat~A Erdogdu.
\newblock An analysis of constant step size {SGD} in the non-convex regime:
  Asymptotic normality and bias.
\newblock \emph{arXiv preprint arXiv:2006.07904v2}, 2020.

\end{thebibliography}
  \bibliographystyle{plainnat}
  \ifdef{\groundskip}
        {\vskip\lastskip
         \addvspace{2\groundskip}}
        {}
}

\ifdef{\groundskip}{}{\clearpage}

\begin{appendices}

\crefalias{section}{appendix}

\section{Proofs for the Analysis of the Algorithm}
\label{sec:proofs-for-analysis}

First, we list a couple of basic properties then restate the propositions from \Cref{sec:tta-analysis} and provide proofs.

\begin{proposition}[Basic properties]
\label{prop:basics}
$\forall t \geq E\colon$ and $\forall n \geq E \colon$
\begin{enumerate}[(i)]
\item $E \divides S(n)$, $E \divides L(n)$\label{prop:multiples}
\item $S(n) \leq L(n) \leq n$\label{prop:slessthanl}
\item $f^S(n) > f^L(n)$\label{prop:shortworse}
\item $L(t) = S(t) + S'\!(\SP(t))$ \,if\, $\SP(t) \neq -1$ \,else\, $L(t) = S(t)$\label{prop:lassumofs}
\end{enumerate}
\end{proposition}

\Cref{prop:multiples} states that the averaging lengths are multiples of the evaluation period; \ref{prop:slessthanl} follows from that the lengths increase by 1 at every iteration except at switches, where $S$ is reset to $0$; \ref{prop:shortworse} is because we switch if it is not true; and \ref{prop:lassumofs} expresses that all long averages except the first are continuations of the previous short average.

\propbounds*

\begin{proof}
We prove $S(n) < \OL(n)$ by contradiction.
Suppose $\OL(n_0) \leq S(n_0)$ for some $n_0$.
As $\OL(n)$ is monotonically increasing and $S(n)$ increases by $E$, there exists $n\leq n_0$ such that $\OL^E(n)=S(n)$.
Since $S(n) \leq L(n)$ (by \ref{prop:slessthanl} of \Cref{prop:basics}), from \Cref{easy-opt2} and $\OL^E(n)=S(n) \leq L(n)$, we have that $f^S(n) \leq f^L(n)$, which contradicts \ref{prop:shortworse} of \Cref{prop:basics}.

Next, we prove $L(n) < 2\OL(n)+E$.
From \ref{prop:lassumofs} of \Cref{prop:basics}, we have that at the beginning, when there has not yet been a switch, $L(t) = S(t)$, else $L(t) = S(t) + S'\!(\SP(t))$ for all $t$.
In the first case, $L(n) = S(n) < \OL(n) < 2\OL(n)+E$, and we are done.

In the second, usual case, $L(n) = S(n) + S'\!(\SP(n))$.
That is, the length of the current long average is the sum of the lengths of the current and the previously finished short average.
Since $S(n) < \OL(n)$ and $S'\!(\SP(n)) = S(\SP(n)-E) + E$, so $L(n) < \OL(n) + S(\SP(n)-E) + E$, from which $L(n) < \OL(n) + \OL(\SP(n)-E) + E$.
Finally, from the monotonicity of $O$ in \Cref{monotone-opt}, $\OL(\SP(n)-E) \leq \OL(n)$, we get $L(n) < 2\OL(n) + E$.
\end{proof}

\propinfinite*

\begin{proof}
Because $S(n) < \OL(n)$ and $S(n)$ increases by $E$ between switch points, it must catch $\OL(n)$ at some step $n_s$ because $\OL(n)$ grows more slowly by \Cref{slow-opt}.
At the point where $S(n_s) = \OL^E(n_s)$, $f^S(n_s) \leq f^L(n_s)$ by \Cref{easy-opt2}, thus there must be a switch.
\end{proof}

\proponce*

\begin{proof}
Since $S(n) < \OL(n)$ and $E \divides S(n)$ for all $n$, so $S'\!(n_1) \leq \OL^E(n_1)$.
Then it is either that $S'\!(n_1) = \OL^E(n_1)$ or $S'\!(n_1) < \OL(n_1)$.
Since at switch points $L(n)=S'\!(n)$, in the former case, we conclude the proof with $L(n_1) = \OL^E(n_1)$.
Considering the latter case, $L(n_1) = S'\!(n_1) < \OL(n_1)$, so $L(n_1) \leq \OL_E(n_1)$.
Also, switches happen when $f^{S'\!\!}(n) \leq f^{L'\!\!}(n)$, but as per \Cref{easy-opt} this can happen only if $\OL(n) < L'\!(n)$.
Thus for $n_2$ to be a switch point, it must be that $\OL(n_2) < L'\!(n_2) = L(n_2-E) + E$, hence $\OL_E(n_2) \leq L(n_2-E)$.
Combining it with $L(n_1) \leq \OL_E(n_1)$, we get $L(n_1) \leq \OL_E(n_1) \leq \OL_E(n_2) \leq L(n_2-E)$.
Therefore, since $L(n)$ and $\OL_E(n)$ are monotonically increasing over $[n_1,n_2-E]$, both take values that are multiples of $E$, and $L$ overtakes $\OL_E$ while not skipping any such value, there must be a point $n$ where $L$ is equal to $\OL_E$.
\end{proof}

\section{Properties of 2TA in the Pure Optimization Setting}
\label{sec:properties-of-tta-in-pure-optimization}

Here, we prove the claim from \Cref{sec:tta-pure-optimization} that \tta{} converges to the optimal weights in the ordinary least squares regression setting, where the generalization and training losses are the same.

\begin{definition}[$N$th switch point]
For all $N \in \Nb$, we define three random variables:
\begin{itemize}
\item $Q_N \in [E, 2E, \dots]$ is the time step corresponding to the $N$th switch point.
\item $F_N = f(\theta^{S'}(Q_N))$ is the loss with the $N$th short average just before it becomes the long average.
\item $S_N = S'(Q_N)$ is the final length of the $N$th short average.
\end{itemize}
\end{definition}

From now on, we use $N$ to index switch points or to refer to short averages that end at that switch point.

\begin{proposition}
\label{prop:f-n-decreasing}
If the generalization loss function $f$ is convex, then $(F_N)_{N \in \Nb}$ is monotonically decreasing.
\end{proposition}
\begin{proof}
The long-averaged weights are a convex combination of the weights of the current and the previous short averages:
\begin{gather*}
\theta^{L'\!\!}(n) = \alpha \theta^{S'\!\!}(n) + (1-\alpha)\theta^{S'\!\!}(\SP(n))\\
\alpha = \frac{S'\!(n)}{S'\!(n) + S'\!(\SP(n))} \in (0,1).
\end{gather*}
Switching happens when $f(\theta^{S'}(n)) \leq f(\theta^{L'}(n))$.
Expanding $\theta^{L'\!\!}(n)$ and using that $f$ is convex, we get
\begin{align*}
f(\theta^{S'\!\!}(n))
&\leq f(\theta^{L'\!\!}(n)) \\
&= f(\alpha \theta^{S'\!\!}(n) + (1-\alpha)\theta^{S'\!\!}(\SP(n))) \\
&\leq \alpha f(\theta^{S'\!\!}(n)) + (1-\alpha)f(\theta^{S'\!\!}(\SP(n))),
\end{align*}
from which, $(1-\alpha)f(\theta^{S'\!\!}(n)) \leq (1-\alpha)f(\theta^{S'\!\!}(\SP(n)))$.
Using $\alpha < 1$, we get $f(\theta^{S'\!\!}(n)) \leq f(\theta^{S'\!\!}(\SP(n)))$.
This is true at all switch points, hence $F_{N+1} \leq F_N$ for all $N$.
\end{proof}

Note that in non-convex settings, the above monotonicity property could be enforced also by changing the switching condition to $f(\theta^{S'\!\!}(n)) \leq \min(f(\theta^{L'\!\!}(n)),\allowbreak f(\theta^{S'\!\!}(\SP(n)))$.
However, this would make the algorithm less able to adapt to violations of \Cref{monotone-opt}.

\begin{proposition}
\label{prop:l-to-infinity}
Assume that the loss function is bounded from below, the $(F_N)_{N \in \Nb}$ sequence monotonically decreases, and that $(\theta_t)_{t \in \natnumzero}$ approaches a stationary distribution with a density.
Then, $L(t) \pto \infty$.
\end{proposition}

\begin{proof}
First, we prove $S_N \pto \infty$ by contradiction.
Assume that $S_N \centernot{\pto} \infty$.
\begin{enumerate}[(i)]
\item \label{convergence-proof-step:indirection}
\emph{For some $l_0 \in \Nb$ and $\epsilon_0 > 0$, there are infinitely many $N^+ \in \Nb$ such that $P(S_{N^+} = l_0) > \epsilon_0$.}\\
\emph{Proof.}
By the definition of convergence in probability, $S_n \pto \infty$ is equivalent to $\forall l \in \Nb \colon \lim_{N \to \infty} P(S_N \leq l) = 0$.
Suppose that is false, hence $\exists l \in \Nb, \epsilon > 0 \colon \forall N \in \Nb \colon \exists N^+ > N \colon P(S_{N^+} \leq l) > \epsilon$.
Then, we have an infinite number of short averages $N^+$ that are at most length $l$ with at least $\epsilon$ probability: $P(S_{N^+} \leq l) > \epsilon$.
Since $l$ is finite, for all such $N^+$, there exists $l_{N^+} \in \Nb$ such that $P(S_{N^+}=l_{N^+})>\epsilon/l$.
Hence, there must be at least one $l_0 \leq l$ and $\epsilon_0 > 0$ such that $P(S_{N^+} = l_0) > \epsilon_0$ for infinitely many $N^+$.

\item \label{convergence-proof-step:loss-convergence}
\emph{The final losses of the short averages converge in probability: $F_N \pto F^*$.}\\
\emph{Proof.}
From the assumption that the loss function is bounded from below and that all realizations of the $F_N$ sequence decrease monotonically, all realizations must converge, which implies almost sure convergence hence convergence in probability.

\item \label{convergence-proof-step:contradiction}
\emph{Let $F^{l_0+}_N$ denote what the loss of $N$th short average at length $l_0$ would be if the algorithm were modified to perform no switching for this short average only.
Then, $P(F^* \leq F^{l_0+}_N \leq F_{N-1}) \to 0$.}\\
\emph{Proof.}
We have assumed that iterates converge to a stationary distribution with a density.
Note that this rules out convergence in the strict sense, which would require a zero-variance stationary distribution.
For any random variable $X$ with a density, $lim_{\delta \to 0} P(a \leq X \leq a+\delta) = 0$ for all $a \in \Rb$.
By \ref{convergence-proof-step:loss-convergence}, $F_N \pto F^*$, so for all $\epsilon > 0$, $P(F_N-F^* < \epsilon)$ is close to $1$ for all large enough $N$.
With $F_{N-1}-F^*$, the size of the interval into which $F^{l_0+}_N$ must fit, thus bounded uniformly in probability, we get $P(F^* \leq \smash{F^{l_0+}_N} \leq F_{N-1}) \to 0$.
\end{enumerate}

We assumed that $S_N \centernot{\pto} \infty$ and in \ref{convergence-proof-step:indirection} showed that $P(S_{N^+} = l_0) > \epsilon_0$ for some $l_0 \in \Nb$, $\epsilon_0 > 0$ and infinitely many $N^+$.
Since $S_N = l_0$ implies $F_N = F^{l_0+}_N$ for any $N$, we have that $P(F_{N^+} = \smash{F^{l_0+}_{N^+}}) > \epsilon_0$.
However, due to the monotonicity assumption, $F^* \leq F_N \leq F_{N-1}$ for all $N$, hence $P(F^* \leq \smash{F^{l_0+}_{N^+}} \leq F_{N^+-1}) > \epsilon_0$, which contradicts $P(\smash{F^*} \leq \smash{F^{l_0+}_N} \leq F_{N-1}) \to 0$ from \ref{convergence-proof-step:contradiction}.

Finally, every long average except the first is a continuation of the previous short average, that is, for all $t \geq E$, $L(t) \geq S'(\SP(t))$ and $S'(\SP(t)) = S_N$ for some $N$.
Therefore, $S_N \pto \infty$ implies that $L(t) \pto \infty$.
\end{proof}

\begin{proposition}
Consider applying SGD with a constant learning rate to an ordinary least squares problem with unique minimum point $\theta^\star$.
Then, for a sufficiently low learning rate, $\theta^L(t) \pto \theta^\star$.
\end{proposition}

\begin{proof}
The loss function is convex, so $F_N$ is monotonically decreasing by \Cref{prop:f-n-decreasing}.
It is also bounded from below, and it satisfies Assumptions 2.1-2.3 of \citet{yu2020analysisv2}, hence --~by Proposition 2 therein~-- SGD iterates admit a unique stationary distribution for an appropriately bounded learning rate.
Thus, appealing to \Cref{prop:l-to-infinity}, we have that $L(t) \pto \infty$.
In the ordinary least squares regression setting, \citet{jain2018parallelizing} prove that Tail Averaging converges to the optimum with an appropriately bounded learning rate.
Hence, by choosing a learning rate that satisfies both bounds and leveraging the fact that $\theta^L(t)$ is a tail average, we get $\theta^L(t) \pto \theta^\star$.
\end{proof}

Paralleling strict convergence results for Polyak and Tail Averaging, we have proved that \tta{} converges in probability to the optimum in the ordinary least squares regression setting when $f$ is the training loss.
We stress again that \tta{} is not intended for pure optimization, and this result is to serve as a characterization of behaviour in the infinite data case.

With pure optimization very much a secondary consideration, we provide only weak, anecdotal support for the rate of convergence: the length of the long average tended to increase exponentially in all experiments described in \Cref{sec:tta-experiments} and also on simple synthetic data.
Intuitively, this is to be expected when $f$ is locally convex because at stationarity, every time a short average finishes at length $l$, it halves the probability mass available for subsequent short averages to finish at that length: $P(\smash{F^l_{N+1}} \leq F_N \mid S_{N+1} \geq l,\ S_N = l) \to 0.5 P(\smash{F^l_N} \leq F_{N-1} \mid S_N \geq l)$.
This halving effect is strongest at the same length, but in diminished form, it extends to longer averages due to the similarity of their distributions.

\end{appendices}

\end{document}